\documentclass{article} 
\usepackage{iclr2027_conference_preprint,times}

\usepackage{amsmath,amsfonts,bm}

\def\eqref#1{equation~\ref{#1}}

\def\1{\bm{1}}

\DeclareMathAlphabet{\mathsfit}{\encodingdefault}{\sfdefault}{m}{sl}
\SetMathAlphabet{\mathsfit}{bold}{\encodingdefault}{\sfdefault}{bx}{n}

\DeclareMathOperator*{\argmin}{arg\,min}

\usepackage{float}
\usepackage{algorithm}
\usepackage{subcaption}
\usepackage{algpseudocode}
\usepackage[dvipsnames]{xcolor}
\definecolor{citegreen}{RGB}{60,160,90}
\definecolor{reflink}{RGB}{160,50,60}
\usepackage{mathtools}
\usepackage{makecell}
\usepackage{multirow}
\usepackage{graphicx}
\usepackage[colorlinks,linkcolor=reflink,urlcolor=black,citecolor=citegreen]{hyperref}
\usepackage{amsmath}
\usepackage{amsthm}
\usepackage{wrapfig}
\usepackage{booktabs}
\usepackage{amssymb}
\usepackage{pifont}
\usepackage{url}
\usepackage{xcolor}
\usepackage{capt-of}
\usepackage{cleveref}
\crefname{section}{Sec.}{Secs.}\Crefname{section}{Sec.}{Secs.}
\crefname{subsection}{Sec.}{Secs.}\Crefname{subsection}{Sec.}{Secs.}
\crefname{appendix}{App.}{Apps.}\Crefname{appendix}{App.}{Apps.}
\crefname{figure}{Fig.}{Figs.}\Crefname{figure}{Fig.}{Figs.}
\crefname{table}{Tab.}{Tabs.}\Crefname{table}{Tab.}{Tabs.}

\definecolor{revblue}{RGB}{0,70,220}
\newif\ifshowrev
\showrevfalse
\long\def\rev#1{\ifshowrev{\color{revblue}#1}\else#1\fi}
\def\revoff{\ifshowrev\color{black}\fi}

\newcommand{\method}{MaDeL}

\title{\method{}: Manifold-Decomposed Feature Losses \\ for Generative Modeling}

\author{
Beomsu Kim$^{1}$ \quad
Jong Chul Ye$^{1}$\thanks{Co-corresponding authors.} \quad
Kwanyoung Kim$^{2}$\footnotemark[2]
\\
$^{1}$Graduate School of AI, KAIST \\
$^{2}$Department of AI, GIST \\
\texttt{\{beomsu.kim, jong.ye\}@kaist.ac.kr} \\
\texttt{k0.kim@gist.ac.kr}
}

\newcommand{\xmark}{\textcolor{red!45!black}{\ding{55}}}
\newcommand{\cmark}{\textcolor{green!45!black}{\ding{51}}}

\renewcommand{\eqref}[1]{Eq.~(\ref{#1})}

\newcommand{\bmg}{{\bm{g}}}
\newcommand{\bmh}{{\bm{h}}}

\newcommand{\bmr}{{\bm{r}}}

\newcommand{\bmu}{{\bm{u}}}
\newcommand{\bmv}{{\bm{v}}}

\newcommand{\bmx}{{\bm{x}}}
\newcommand{\bmy}{{\bm{y}}}
\newcommand{\bmz}{{\bm{z}}}

\newcommand{\bmtheta}{{\bm{\theta}}}

\newcommand{\bmphi}{{\bm{\phi}}}

\newcommand{\bmA}{{\bm{A}}}
\newcommand{\bmC}{{\bm{C}}}
\newcommand{\bmD}{{\bm{D}}}
\newcommand{\bmE}{{\bm{E}}}
\newcommand{\bmF}{{\bm{F}}}
\newcommand{\bmG}{{\bm{G}}}
\newcommand{\bmI}{{\bm{I}}}
\newcommand{\bmJ}{{\bm{J}}}
\newcommand{\bmP}{{\bm{P}}}
\newcommand{\bmR}{{\bm{R}}}
\newcommand{\bmU}{{\bm{U}}}

\newcommand{\cL}{\mathcal{L}}
\newcommand{\cM}{\mathcal{M}}
\newcommand{\cN}{\mathcal{N}}

\newcommand{\bbE}{\mathbb{E}}

\newcommand{\bbR}{\mathbb{R}}

\newcommand{\ie}{\textit{i.e.}}

\DeclareMathOperator{\Image}{Im}
\DeclareMathOperator{\Span}{Span}
\DeclareMathOperator{\diag}{diag}
\DeclareMathOperator{\rank}{rank}

\DeclareMathOperator{\sg}{sg}
\DeclareMathOperator{\loc}{loc}
\DeclareMathOperator{\iso}{iso}
\DeclareMathOperator{\targ}{targ}

\let\ae\relax
\DeclareMathOperator{\ae}{AE}

\newtheorem{proposition}{Proposition}

\iclrpreprintcopy
\begin{document}

\setcounter{footnote}{1}
\maketitle
\setcounter{footnote}{0}

\begin{figure}[H]
\vspace{-6mm}
\centering
\includegraphics[width=1.0\linewidth]{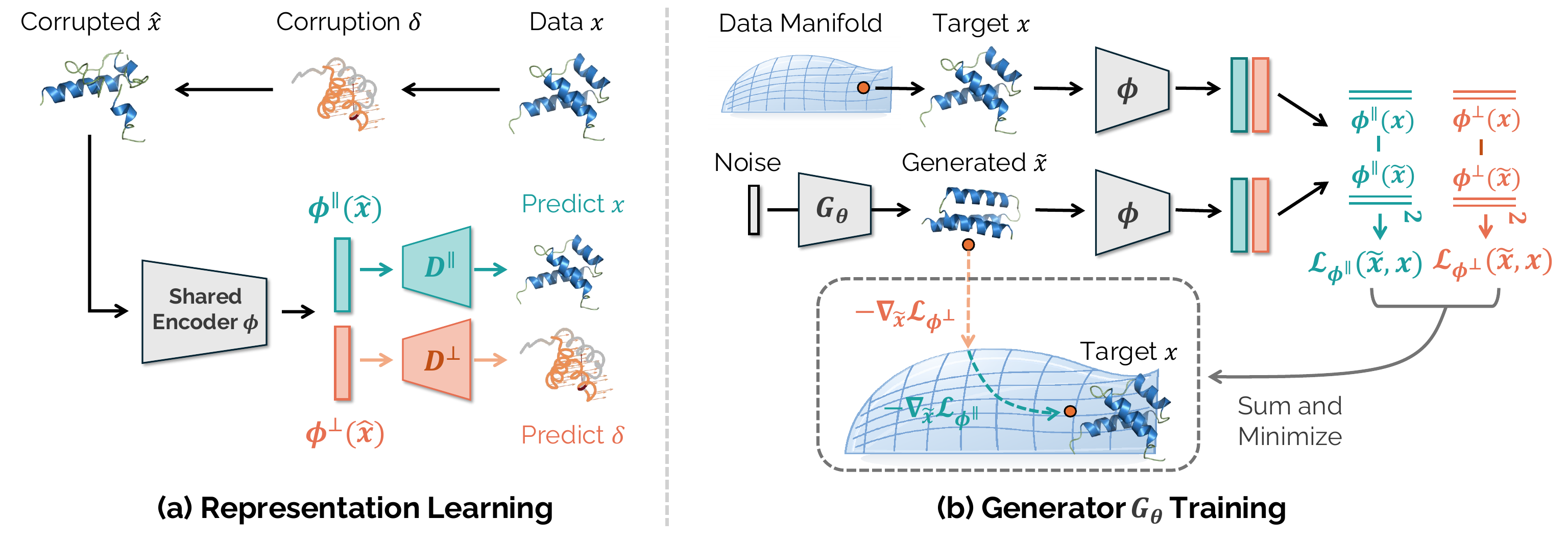}
\vspace{-4mm}
\caption{\textbf{Overview of \method{}.} \textbf{(a)} A shared encoder splits a corrupted observation into manifold-parallel and manifold-orthogonal features, trained with $\bmx$ and $\bm{\delta}$-prediction. \textbf{(b)} These learned features define the decomposed generative loss. The resulting gradients produce complementary tangent and normal updates to the data manifold.}
\label{fig:method_sketch}
\end{figure}

\begin{abstract}

Generative models are often trained with isotropic objectives such as mean-squared error.
For data concentrated near a low-dimensional manifold, however, such losses conflate displacement along the manifold, which may represent valid variation, with displacement away from it, which produces invalid samples.
This mismatch is especially problematic in sparse, highly constrained domains, where ambient-space regression can encourage off-manifold interpolation.
\emph{We ask whether a generative objective can distinguish manifold-parallel variation from manifold-orthogonal deviation directly from data, without explicitly estimating the manifold.}
We introduce a manifold-decomposed feature loss (\method{}) that learns complementary representations from corrupted observations: one is trained to recover the clean sample, while the other is trained to recover the corruption.
We show that, under a feature bottleneck, their Jacobians align with the tangent and normal spaces, exactly for linear manifolds and locally for smooth manifolds.
Together, these representations define an anisotropic objective that separately measures intrinsic variation and off-manifold deviation.
Across synthetic, Earth and climate science, and torsion-angle benchmarks, \method{} improves support recovery and average angular $W_1$ under single-step sampling; on protein backbones, it reduces steric clashes across one- and few-step sampling budgets.
\end{abstract}

\vspace{-5mm}
\section{Introduction} \label{sec:intro}

Generative models have achieved remarkable fidelity across images, video, molecules, and protein structures~\citep{karras2022edm,ho2022video,hoogeboom2022equivariant,yim2023se3}.
Despite substantial differences in architecture and formulation, many modern approaches share a common training principle: they learn a denoising, score, velocity, or reconstruction target using a largely isotropic regression objective.
For example, diffusion and flow-based models commonly minimize an $\ell_2$ discrepancy between predicted and target quantities~\citep{ho2020ddpm,song2021score,lipman2023flow,liu2023rectified}.
Such objectives measure \emph{how far} a prediction is from its target, but are largely agnostic to \emph{the direction in which the prediction is wrong}.

This isotropic treatment is poorly matched to the geometry of natural data.
High-dimensional observations are hypothesized to concentrate near a lower-dimensional, structured subset of the ambient space~\citep{tenenbaum2000isomap,narayanan2010manifold,fefferman2016manifold,pope2021intrinsic}.
Locally, perturbations of equal magnitude can have very different consequences depending on their orientation relative to the data manifold~\citep{rifai2011manifold,alain2014regularized}.
A displacement parallel to the manifold may correspond to meaningful variation, such as a change in object pose or a valid protein conformation, whereas an orthogonal displacement can produce perceptual artifacts or physically invalid structures.
An isotropic loss cannot distinguish between them.

This distinction is especially important for generative modeling.
Near the data manifold, a generator must both suppress deviations away from valid data and model variation along the manifold.
The desired dynamics are therefore inherently anisotropic: \emph{orthogonal directions should restore samples toward the manifold, whereas parallel directions should describe variation within it}.
Recent analyses of diffusion models support this geometric perspective, showing that score functions recover directions associated with the normal bundle of the data manifold at low noise levels~\citep{stanczuk2024diffusion}.
Ignoring this distinction encourages interpolation through off-manifold, low-density regions.

Ideally, a generative objective would therefore measure prediction errors according to the local geometry of the data distribution rather than a fixed isotropic metric.
Explicitly estimating tangent and normal spaces for complex high-dimensional data, however, is challenging.
Instead, we use learned feature spaces to induce data-dependent notions of similarity, building on their successful use as perceptual losses for image synthesis~\citep{johnson2016perceptual,zhang2018unreasonable}.

Concretely, given data $\bmx$, we construct a corrupted observation $\tilde{\bmx}=\bmx+\bm\delta$ and learn two complementary representations.
\emph{Manifold-parallel features} are learned by predicting $\bmx$ from $\tilde{\bmx}$ (\emph{$\bmx$-prediction}), encouraging them to preserve intrinsic variation while suppressing corruption.
\emph{Manifold-orthogonal features} are learned by predicting $\bm\delta$ from the same observation (\emph{$\bm\delta$-prediction}), encouraging them to capture deviations away from the manifold.
Together, these representations provide a learned decomposition of errors into manifold-parallel and orthogonal components.

We use these representations to define a geometry-aware training objective for generative models: rather than penalizing all prediction errors uniformly, it measures them through feature spaces specialized for intrinsic variation and off-manifold deviation, and thus distinguishes predictions that are equally distant in ambient space but differ in their compatibility with the data distribution (\Cref{fig:method_sketch}).

Specifically, our contributions are:
\begin{itemize}
    \item \textbf{Theoretical characterization of learned geometry (\Cref{sec:theory_linear,sec:theory_smooth}).}
    We show that, under suitable bottlenecks, optimal $\bmx$-prediction and $\bm\delta$-prediction encoders have Jacobian row spaces contained in the tangent and normal spaces of the data manifold, respectively, exactly for linear manifolds and to leading order locally for smooth manifolds. Their feature losses yield complementary manifold-parallel and manifold-orthogonal gradients.
    \item \textbf{Manifold-decomposed generative loss (\Cref{sec:autoencoder_to_loss}).}
    We introduce \method{}, a geometry-aware feature loss that separately measures manifold-parallel variation and manifold-orthogonal deviation using these representations, without access to the data manifold.
    \item \textbf{Validation across diverse manifolds (\Cref{sec:experiments}).}
    We evaluate \method{} on a helix, Earth and climate science data, protein and RNA torsion angles, and protein backbone generation. \method{} improves support recovery and average angular $W_1$ on the single-step benchmarks and reduces steric clashes in protein backbones across one- and few-step budgets.
\end{itemize}

\vspace{-3mm}
\section{Background} \label{sec:background}
\vspace{-2mm}

\paragraph{Isotropic regression objectives.}
Let $\hat{\bmx}_{\bmtheta}$ and $\bmx_{\mathrm{targ}}$ denote a model prediction and its regression target. Diffusion and Consistency Models~\citep{song2023consistency,song2024icm}, for example, are commonly trained with isotropic ambient-space objectives of the form
\begin{align}
    \textstyle \cL_{\iso}(\bmtheta)
    \coloneqq
    \bbE[\rho(\|\hat{\bmx}_{\bmtheta}-\bmx_{\mathrm{targ}}\|_2^2)],
\end{align}
where $\rho(u) = u$ yields mean-squared error (MSE) and $\rho(u)=\sqrt{u+c^2}-c$ yields the pseudo-Huber loss.
Although these methods construct $\bmx_{\mathrm{targ}}$ differently, their losses depend only on the residual magnitude and therefore penalize equal-magnitude errors identically regardless of direction.

\paragraph{Geometry-aware generative modeling.}
Natural data often concentrate near low-dimensional subsets of the ambient space, making their local geometry anisotropic: perturbations of equal magnitude can represent meaningful variation when tangent to the data manifold but invalid deviation when normal to it, a distinction ignored by isotropic regression.

\begin{wraptable}{r}{0.53\linewidth}
\vspace{-2mm}
\centering
\scriptsize
\setlength{\tabcolsep}{2.5pt}
\begin{tabular}{lccc}
\toprule
\textbf{Method} & Geometry & Tan.\,/\,Nor. & Train. \\
\midrule
RFM~\citep{chen2024flow} & given & \cmark/\xmark & \cmark \\
RMF~\citep{woo2026riemannianmeanflow} & given & \cmark/\xmark & \cmark \\
Tango-DM~\citep{liu2025tango} & given & \cmark/\xmark & \cmark \\
SDDM~\citep{sun2023sddm} & pretrained & \cmark/\cmark & \xmark \\
MPGD~\citep{he2024mpgd} & pretrained & \xmark/\cmark & \xmark \\
Saito \& Matsubara~\citeyearpar{saito2025tangential} & pretrained & \cmark/\cmark & \xmark \\
AYT~\citep{kim2026ayt} & learned & \xmark/\cmark & \cmark \\
\midrule
\method{} (Ours) & learned & \cmark/\cmark & \cmark \\
\bottomrule
\end{tabular}
\vspace{-0.5em}
\caption{Geometry-aware methods: source of geometry, tangent/normal modeling, and use in training.}
\label{tab:related}
\vspace{-5mm}
\end{wraptable}
\Cref{tab:related} summarizes how prior work accounts for this geometry.
Existing approaches primarily incorporate the geometry of a known state space: Riemannian diffusion models~\citep{huang2022riemannian} and Riemannian Flow Matching (RFM)~\citep{chen2024flow} extend generative dynamics to non-Euclidean manifolds, while MeanFlow (MF)~\citep{geng2025meanflow} and Riemannian MeanFlow (RMF)~\citep{woo2026riemannianmeanflow} enable one-step generation in Euclidean and manifold-valued spaces, respectively, and Tango-DM trains only the tangential score component of a known manifold~\citep{liu2025tango}.
Other work targets the geometry of the embedded data distribution, but extracts it from a pretrained model: low-noise diffusion scores approximate its normal bundle~\citep{stanczuk2024diffusion}, SDDM decomposes scores into tangent and normal components~\citep{sun2023sddm}, MPGD constructs manifold-preserving guidance from pretrained autoencoders~\citep{he2024mpgd}, and the score Jacobian yields a tangent-aware metric~\citep{saito2025tangential}.
In contrast, we learn complementary tangent- and normal-sensitive features from data and place them directly in the training objective, with neither a known manifold nor a pretrained models.

\paragraph{Feature losses as learned directional metrics.}
Learned features introduce anisotropy without an explicit parameterization of the data manifold: perceptual losses measure errors in representations aligned with perceptual similarity~\citep{johnson2016perceptual,zhang2018unreasonable}, and regularized autoencoders capture local geometry through denoising, Jacobian contraction, or reconstruction fields~\citep{vincent2008extracting,rifai2011contractive,alain2014regularized}. Features have also been incorporated directly into generative training through pretrained discriminator features~\citep{sauer2021projected} and feature-based consistency objectives~\citep{song2023consistency}, including manifold-aligned tangents (AYT)~\citep{kim2026ayt}. These approaches, however, provide a single notion of similarity rather than complementary sensitivities to intrinsic variation and off-manifold deviation.

\begin{figure}[t]
\centering
\includegraphics[width=1.0\linewidth]{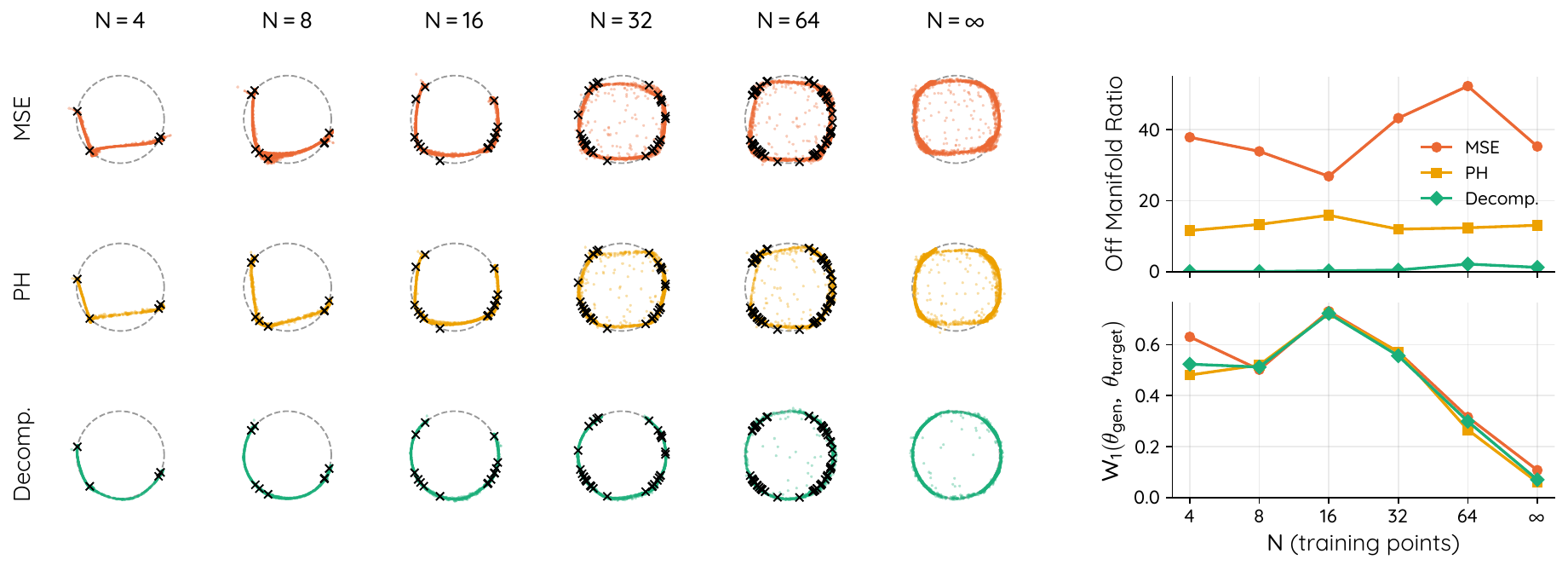}
\vspace{-6mm}
\caption{\textbf{A manifold-decomposed metric mitigates off-manifold interpolation.} With sparse training samples (crosses), Cartesian MSE and Pseudo-Huber interpolate through the circle, whereas the oracle decomposed metric preserves its support.}
\vspace{-4mm}
\label{fig:toy_sweep}
\end{figure}

\vspace{-2mm}
\section{Motivation} \label{sec:motivation}

Let $\bmphi:\bbR^d\rightarrow\bbR^n$ be a feature encoder and define $\cL_{\bmphi}(\hat{\bmx},\bmx)\coloneqq\frac{1}{2} \|\bmphi(\hat{\bmx})-\bmphi(\bmx)\|_2^2$. Its gradient is $\bmJ_{\bmphi}(\hat{\bmx})^\top(\bmphi(\hat{\bmx})-\bmphi(\bmx))$, so the encoder Jacobian determines the input-space directions to which the loss is sensitive. We learn complementary encoders $\bmphi^\parallel$ and $\bmphi^\perp$ for manifold-parallel variation and manifold-orthogonal deviation, respectively. Writing $a\in{\parallel,\perp}$ for the two branches, we define
\begin{align}
\cL(\hat{\bmx},\bmx)
&=\textstyle \sum_{a\in{\parallel,\perp}}\lambda^a
\cL_{\bmphi^a}(\hat{\bmx},\bmx),
\label{eq:decomposed_loss} \\
\nabla_{\hat{\bmx}}\cL(\hat{\bmx},\bmx)
&=\textstyle \sum_{a\in{\parallel,\perp}}\lambda^a
\bmJ_{\bmphi^a}(\hat{\bmx})^\top
\left(\bmphi^a(\hat{\bmx})-\bmphi^a(\bmx)\right).
\label{eq:feature_gradient}
\end{align}
If the encoders $\bmphi^\parallel$ and $\bmphi^\perp$ specialize as intended, these terms provide complementary update directions along and away from the data manifold, inducing a learned directional geometry rather than a generic perceptual metric.

\paragraph{Sanity check on a toy example.}

We first isolate the effect of the training metric on a manifold that is known exactly, using an oracle decomposition that the learned features of \Cref{sec:method} are designed to approximate. The target distribution is a mixture of four von Mises distributions on the unit circle in $\bbR^2$. For each dataset size $N$, identical generators are trained with the same OT-matched objective and differ only in the regression metric; the crosses in \Cref{fig:toy_sweep} denote the training samples, and $N=\infty$ denotes training with fresh samples from the underlying distribution.\revoff

As shown in \Cref{fig:toy_sweep} left, Cartesian MSE and Pseudo-Huber (PH) interpolate between sparse modes along Euclidean chords and produce off-manifold samples inside the circle. The oracle decomposed metric applies PH in polar coordinates with $(\lambda^\parallel,\lambda^\perp) = (1,5)$, where angle and radius correspond to the manifold-parallel and manifold-orthogonal directions, and preserves the support.

\Cref{fig:toy_sweep} right quantifies this behavior: the decomposed metric reduces the fraction of off-manifold samples to nearly zero across dataset sizes while attaining a comparable angular $W_1$ distance to the target distribution. The improvement therefore arises from the manifold-orthogonal geometry without loss of variation along the manifold, which underscores the role of the training metric.

\vspace{-2mm}
\section{Main Contribution: \method{}} \label{sec:method}

Building on \Cref{sec:motivation}, we show that the two feature spaces arise from $\bmx$-prediction and $\bm\delta$-prediction on corrupted data, and turn them into a two-stage training procedure. We characterize the optimal encoders exactly for linear manifolds (\Cref{sec:theory_linear}), extend the characterization locally to smooth manifolds (\Cref{sec:theory_smooth}), connect the encoder Jacobians to the update directions of the decomposed loss, and describe the shared encoder and generator training (\Cref{sec:autoencoder_to_loss}).

\subsection{Theory: Linear Manifolds and Autoencoders} \label{sec:theory_linear}

Suppose the data $\bmx$ is supported on an $r$-dimensional linear subspace of $\bbR^d$. Concretely,
\begin{align}
    \bmx &= \bmA \bmz, \quad \bmz \sim \cN(\bm{0},\bmI), \quad r \coloneqq \rank(\bmA) < d,
\end{align}
such that data $\bmx$ is Gaussian distributed with covariance
\begin{align}
    \bm{\Sigma} \coloneqq \bmA \bmA^\top = \bmU \diag(\lambda_1,\ldots,\lambda_d) \bmU^\top,
\end{align}
with eigenvalues $\lambda_1 \geq \cdots \geq \lambda_r > 0 = \lambda_{r+1} = \cdots = \lambda_d$.
We consider linear encoder-decoder pairs $(\bmE,\bmD)$ with bottleneck dimensions $n_\parallel$ for $\bmx$-prediction and $n_\perp$ for $\bm\delta$-prediction, together with isotropic corruption $\bm{\delta} \sim \cN(\bm{0},\sigma^2\bmI)$.

\begin{proposition}[$\bmx$-prediction selects high-variance directions]
\label{prop:linear_parallel}
Consider the denoising objective
\begin{align}
    \textstyle \min_{\bmE,\bmD} \bbE_{\bmx,\bm{\delta}} \| \bmx - \bmD \bmE (\bmx + \bm{\delta}) \|_2^2, \quad n_\parallel \leq r.
\end{align}
Assume there are no ties between positive eigenvalues, \ie, $\lambda_{n_\parallel} > \lambda_{n_\parallel+1}$, and let $\bmU_{n_\parallel}$ contain the eigenvectors corresponding to the $n_\parallel$ largest eigenvalues of $\bm{\Sigma}$. Then the optimal composite map is
\begin{gather}
    \textstyle \bmD^\star \bmE^\star = \bmU_{n_\parallel} \bm{\Gamma}_{\parallel} \bmU_{n_\parallel}^\top, \quad  \bm{\Gamma}_{\parallel} \coloneqq \diag (\frac{\lambda_1}{\lambda_1+\sigma^2}, \ldots, \frac{\lambda_{n_\parallel}}{\lambda_{n_\parallel}+\sigma^2}).
\end{gather}
Moreover, every optimal rank-$n_\parallel$ factorization can be written as $\bmE^\star = \bmR \bmU_{n_\parallel}^\top$, $\bmD^\star = \bmU_{n_\parallel} \bm{\Gamma}_{\parallel} \bmR^{-1}$
for an arbitrary invertible matrix $\bmR \in \bbR^{n_\parallel \times n_\parallel}$.
\end{proposition}

Proposition~\ref{prop:linear_parallel} shows that $\bmx$-prediction allocates the limited representation capacity to directions of largest data variance.
Importantly, the encoder is insensitive to every direction outside the selected principal subspace: $\bmE^\star \bmv = \bm{0}$ for all $\bmv \perp \Span(\bmU_{n_\parallel})$.

For the linear data manifold, $\ker(\bm{\Sigma})$ is exactly its normal space; hence, all manifold-orthogonal directions are annihilated by the $\bmx$-prediction encoder.
The corruption in the input is essential for this result: unlike ordinary autoencoding of $\bmx$ from $\bmx$, denoising exposes the encoder to off-manifold directions while requiring the output to remain unchanged.

\begin{proposition}[$\bm\delta$-prediction selects low-variance directions]
\label{prop:linear_orthogonal}
Consider instead the objective
\begin{align}
    \textstyle \min_{\bmE,\bmD} \bbE_{\bmx,\bm{\delta}} \| \bm{\delta} - \bmD \bmE(\bmx+\bm{\delta}) \|_2^2.
\end{align}
Let $\bmU_{-n_\perp}$ contain eigenvectors corresponding to the $n_\perp$ smallest eigenvalues $\lambda_{d-n_\perp+1},\ldots,\lambda_d$ of $\bm{\Sigma}$. Then the optimal composite map is
\begin{gather}
    \textstyle \bmD^\star \bmE^\star = \bmU_{-n_\perp} \bm{\Gamma}_{\perp} \bmU_{-n_\perp}^\top, \quad \bm{\Gamma}_{\perp} \coloneqq \diag(\frac{\sigma^2}{\lambda_{d-n_\perp+1}+\sigma^2}, \ldots, \frac{\sigma^2}{\lambda_d+\sigma^2}),
\end{gather}
and every optimal rank-$n_\perp$ factorization can be written as $\bmE^\star = \bmR \bmU_{-n_\perp}^\top$, $\bmD^\star = \bmU_{-n_\perp} \bm{\Gamma}_{\perp} \bmR^{-1}$ for an arbitrary invertible matrix $\bmR \in \bbR^{n_\perp \times n_\perp}$.
\end{proposition}

In contrast to $\bmx$-prediction, $\bm\delta$-prediction prioritizes directions of smallest data variance. In particular, if $n_\perp \leq d-r$, the selected directions can lie entirely within $\ker(\bm{\Sigma})$, which is exactly the normal space of the data manifold. When $n_\perp=d-r$, the encoder spans the complete normal space. Thus, under a suitable bottleneck, the two objectives induce complementary representations: $\bmx$-prediction selects manifold-parallel variation, whereas $\bm\delta$-prediction selects manifold-orthogonal variation.

The distinction can also be seen directly from their direction-wise gains. For an eigen-direction with data variance $\lambda$, $\bmx$-prediction and $\bm\delta$-prediction assign $g_{\parallel}(\lambda) = \frac{\lambda}{\lambda+\sigma^2}$ and $g_{\perp}(\lambda) = \frac{\sigma^2}{\lambda+\sigma^2}$
which satisfy $g_{\parallel}(\lambda) + g_{\perp}(\lambda) = 1$.
$\bmx$-prediction therefore increasingly favors high-variance directions, while $\bm\delta$-prediction increasingly favors low-variance directions.

\subsection{Theory: General Smooth Manifolds and Autoencoders} \label{sec:theory_smooth}

We now extend the linear analysis to a general smooth data manifold with smooth autoencoder.
Let $\cM \subset \bbR^d$ be an $r$-dimensional smooth embedded manifold and consider a point $\bmx_0 \in \cM$.
We parameterize nearby clean samples on the manifold as $\bmx_\epsilon = \iota(\exp_{\bmx_0}(\epsilon \bm{\xi}))$, where the random tangent vector $\bm{\xi} \in T_{\bmx_0}\cM$ possesses the first and second moments
\begin{align}
    \bbE[\bm{\xi}] = \bm{0}, \quad \bbE[\bm{\xi}\bm{\xi}^\top] = \bm{\Sigma}_T,
\end{align}
and $\iota : \cM \rightarrow \bbR^d$ denotes the embedding. Let $\bmJ_\iota \coloneqq d\iota_{\bmx_0}$ and
\begin{align}
    \bmC_T \coloneqq \bmJ_\iota \bm{\Sigma}_T \bmJ_\iota^\top = \bmU \diag(\lambda_1,\ldots,\lambda_d) \bmU^\top, \quad \lambda_1 \geq \cdots \geq \lambda_r > 0 = \lambda_{r+1}=\cdots=\lambda_d,
\end{align}
so that $\Image(\bmC_T)=T_{\bmx_0}\cM$ and $\ker(\bmC_T)=N_{\bmx_0}\cM$ whenever $\bm{\Sigma}_T$ has full rank on $T_{\bmx_0}\cM$. As before, independent isotropic corruption $\bm{\delta}_\sigma \sim \cN(\bm{0},\sigma^2 \bmI)$ is added to $\bmx_\epsilon$.

For each objective, let the encoder $\bmE$ and decoder $\bmD$ be twice continuously differentiable, with bottleneck dimension $n_\parallel$ for $\bmx$-prediction and $n_\perp$ for $\bm\delta$-prediction.
Denote their local Jacobians by
\begin{align}
    \bmE_1 \coloneqq d\bmE_{\bmx_0}, \quad \bmD_1 \coloneqq d\bmD_{\bmE(\bmx_0)}.
\end{align}
The key observation is that, after first-order expansion around $\bmx_0$, both objectives reduce exactly to the corresponding linear problems addressed in \Cref{sec:theory_linear}, with the local tangent covariance $\bmC_T$ replacing the global covariance $\bm{\Sigma}$.

\begin{proposition}[Local behavior of $\bmx$-prediction]
\label{prop:general_parallel}
Assume $\bmD(\bmE(\bmx_0))=\bmx_0$. As $\epsilon , \sigma \rightarrow 0$,
\begin{align}
    \bbE_{\bm{\xi},\bm{\delta}_\sigma} \| \bmx_\epsilon - \bmD(\bmE(\bmx_\epsilon+\bm{\delta}_\sigma)) \|_2^2 =
    \bbE_{\bm{\xi},\bm{\delta}_\sigma} \| \epsilon \bmJ_\iota \bm{\xi} - \bmD_1 \bmE_1(\epsilon \bmJ_\iota\bm{\xi} + \bm{\delta}_\sigma) \|_2^2 + o(\epsilon^2+\sigma^2).
\end{align}
For $n_\parallel\rev{\leq} r$, the leading-order objective is minimized by
\begin{gather}
    \textstyle \bmD_1^\star \bmE_1^\star = \bmU_{n_\parallel} \bm{\Gamma}_{\parallel}^{\loc} \bmU_{n_\parallel}^\top, \quad \bm{\Gamma}_{\parallel}^{\loc} \coloneqq \diag(\frac{\epsilon^2\lambda_1}{\epsilon^2\lambda_1+\sigma^2}, \ldots, \frac{\epsilon^2\lambda_{n_\parallel}}{\epsilon^2\lambda_{n_\parallel}+\sigma^2}),
\end{gather}
with $\bmE_1^\star = \bmR \bmU_{n_\parallel}^\top$, $\bmD_1^\star = \bmU_{n_\parallel} \bm{\Gamma}_{\parallel}^{\loc} \bmR^{-1}$ for an arbitrary invertible matrix $\bmR \in \bbR^{n_\parallel \times n_\parallel}$.
\end{proposition}

Thus, to leading order, the Jacobian of the $\bmx$-prediction encoder lies in the local tangent space and is insensitive to normal perturbations: $\bmE_1^\star \bmv = 0$ for all $\bmv \in N_{\bmx_0}\cM$.

\begin{proposition}[Local behavior of $\bm\delta$-prediction]
\label{prop:general_orthogonal}
Assume $\bmD(\bmE(\bmx_0)) = \bm{0}$. As $\epsilon,\sigma\rightarrow 0$,
\begin{align}
    \bbE_{\bm{\xi},\bm{\delta}_\sigma} \| \bm{\delta}_\sigma - \bmD(\bmE(\bmx_\epsilon+\bm{\delta}_\sigma)) \|_2^2 = \bbE_{\bm{\xi},\bm{\delta}_\sigma}
    \| \bm{\delta}_\sigma - \bmD_1 \bmE_1(\epsilon \bmJ_\iota\bm{\xi}+\bm{\delta}_\sigma) \|_2^2 + o(\epsilon^2+\sigma^2).
\end{align}
The leading-order objective is minimized by
\begin{align}
    \textstyle \bmD_1^\star \bmE_1^\star = \bmU_{-n_\perp} \bm{\Gamma}_{\perp}^{\loc} \bmU_{-n_\perp}^\top, \quad \bm{\Gamma}_{\perp}^{\loc} \coloneqq \diag(\frac{\sigma^2}{\epsilon^2\lambda_{d-n_\perp+1}+\sigma^2}, \ldots, \frac{\sigma^2}{\epsilon^2\lambda_d+\sigma^2}),
\end{align}
with $\bmE_1^\star = \bmR \bmU_{-n_\perp}^\top$, $\bmD_1^\star = \bmU_{-n_\perp} \bm{\Gamma}_{\perp}^{\loc} \bmR^{-1}$ for an arbitrary invertible matrix $\bmR \in \bbR^{n_\perp \times n_\perp}$.
\end{proposition}

\begin{wrapfigure}{r}{0.36\linewidth}
\centering
\vspace{-6mm}
\includegraphics[width=1.0\linewidth]{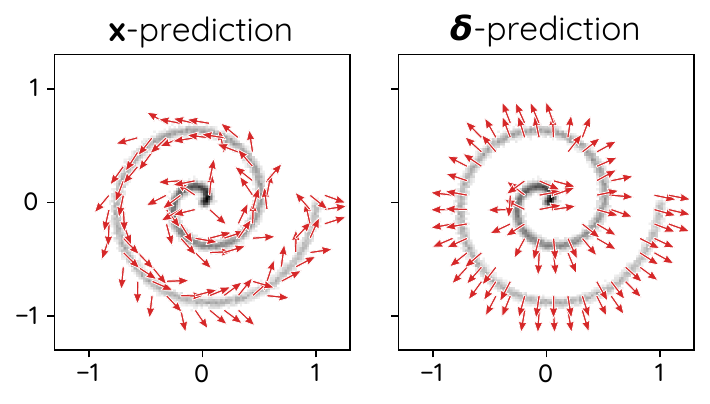}
\vspace{-8mm}
\caption{\textbf{Encoder Jacobians on a 2D spiral.}}
\vspace{-6mm}
\label{fig:toy_jacobian}
\end{wrapfigure}

For $n_\perp\leq d-r$, the vanishing eigenvalues of $\bmC_T$ correspond to normal directions; hence, the local Jacobian of the $\bm\delta$-prediction encoder lies in the normal space and spans it when $n_\perp=d-r$.
Propositions~\ref{prop:general_parallel} and~\ref{prop:general_orthogonal} thus provide complementary local sensitivities without access to the tangent or normal spaces.
\Cref{fig:toy_jacobian} verifies this on a two-dimensional spiral, where the $\bmx$- and $\bm\delta$-prediction encoders have Jacobians parallel and orthogonal to the manifold.

The linear setting is therefore not merely illustrative: after locally linearizing the manifold and encoder-decoder, the general problem reduces to the rank-constrained regression of \Cref{sec:theory_linear}.
This motivates the construction below, where $\bmx$- and $\bm\delta$-prediction encoders define the manifold-parallel and manifold-orthogonal losses.

\subsection{From Autoencoders to Generative Learning}
\label{sec:autoencoder_to_loss}

We now connect these local Jacobians to the generative objective in \eqref{eq:decomposed_loss}. Recall from \eqref{eq:feature_gradient} that the gradient of a feature loss is obtained by mapping the feature residual back to the ambient space through the transpose of the feature Jacobian; the row space of the encoder Jacobian therefore determines the directions along which the loss can update a generated sample.

\begin{proposition}[Directional decomposition of the feature loss]
\label{prop:generative_gradient_decomposition}
Let $\bmx_0 \in \cM$, $n_\perp\leq d-r$, and $\bmphi^\parallel$ and $\bmphi^\perp$ denote encoders satisfying the locally optimal characterizations of Prop.~\ref{prop:general_parallel} and~\ref{prop:general_orthogonal}. Then
\begin{align}
    \Image((\bmJ_{\bmphi^\parallel}(\bmx_0))^\top) \subseteq T_{\bmx_0}\cM, \quad \Image((\bmJ_{\bmphi^\perp}(\bmx_0))^\top) \subseteq N_{\bmx_0}\cM.
\end{align}
Thus, the gradient of the decomposed feature loss in \eqref{eq:decomposed_loss}, evaluated locally at $\bmx_0$, decomposes
\begin{align}
    \textstyle \nabla_{\hat{\bmx}} \cL = \lambda^\parallel \bmg^\parallel + \lambda^\perp \bmg^\perp, \quad \bmg^\parallel \in T_{\bmx_0}\cM, \quad
    \bmg^\perp \in N_{\bmx_0}\cM.
\end{align}
\end{proposition}
\Cref{prop:generative_gradient_decomposition} connects the representation-learning objectives to the generative loss: the $\bmx$-prediction encoder yields feature gradients confined to manifold-parallel directions and the $\bm\delta$-prediction encoder to manifold-orthogonal directions, so that their combination in \eqref{eq:decomposed_loss} updates a generated sample along complementary directions.

\begin{algorithm}[t]
\caption{\method{}: Learning Manifold-Decomposed Features and the Generative Model}
\label{alg:manifold_decomposed_training}
\begin{algorithmic}[1]
\Require Encoder $\bmphi$, projections $\bmP_{\parallel},\bmP_{\perp}$ onto its first $n_\parallel$ and last $n_\perp$ features, decoders $\bmD^{\parallel},\bmD^{\perp}$, generator $\bmG_\theta$; steps $N_{\mathrm{rep}},N_{\mathrm{gen}}$; weights $(\lambda_{\bmx},\lambda_{\bm{\delta}},\lambda_{\iso},\lambda^{\parallel},\lambda^{\perp})$
\Statex \textbf{Stage 1: Representation learning}
\For{step $n=1,\ldots,N_{\mathrm{rep}}$}
    \State Sample a clean sample $\bmx$, corruption $\bm{\delta}$, and let $\tilde{\bmx}\gets\bmx+\bm{\delta}$
    \State Compute representations
    $\bmh^{\parallel}\gets(\bmP_{\parallel}\circ\bmphi)(\tilde{\bmx})$ and
    $\bmh^{\perp}\gets(\bmP_{\perp}\circ\bmphi)(\tilde{\bmx})$
    \State Minimize $\lambda_{\bmx} \cdot \|\bmx-\bmD^{\parallel}(\bmh^{\parallel})\|_2^2
    +\lambda_{\bm{\delta}} \cdot \|\bm{\delta}-\bmD^{\perp}(\bmh^{\perp})\|_2^2$ w.r.t. $(\bmphi,\bmD^{\parallel},\bmD^{\perp})$
\EndFor
\State Freeze
$\bmphi^{\parallel}\coloneqq\bmP_{\parallel}\circ\bmphi$
and
$\bmphi^{\perp}\coloneqq\bmP_{\perp}\circ\bmphi$

\Statex \textbf{Stage 2: Generative model learning}
\For{step $n=1,\ldots,N_{\mathrm{gen}}$}
    \State Sample prior $\bm{\varepsilon}$ and generative model output $\hat{\bmx}_\theta \gets \bmG_\theta(\bm{\varepsilon})$
    \State Assign a target $\bmx_{\targ}$ to $\hat{\bmx}_\theta$ by the underlying generative method
    \State Minimize
    $\lambda_{\iso} \|\hat{\bmx}_\bmtheta - \bmx_{\targ}\|_2^2 + \lambda^{\parallel} \| \bmphi^{\parallel}(\hat{\bmx}_{\theta}) -\bmphi^{\parallel}(\bmx_{\targ})\|_2^2
    +\lambda^{\perp} \|\bmphi^{\perp}(\hat{\bmx}_{\theta}) -\bmphi^{\perp}(\bmx_{\targ})\|_2^2$ w.r.t. $\theta$
\EndFor
\end{algorithmic}
\end{algorithm}

\paragraph{Learning the encoders.}
We use a shared-encoder dual-decoder architecture. Given $\tilde{\bmx}=\bmx+\bm\delta$, the encoder $\bmphi:\bbR^d\rightarrow\bbR^n$, with $n=n_\parallel+n_\perp$, is split by coordinate projections $\bmP_\parallel$ and $\bmP_\perp$ into
\begin{align}
    \bmphi^\parallel \coloneqq \bmP_\parallel\circ\bmphi,
    \qquad
    \bmphi^\perp \coloneqq \bmP_\perp\circ\bmphi.
\end{align}
These branch encoders correspond to Propositions~\ref{prop:general_parallel} and~\ref{prop:general_orthogonal}. Under the assumptions of \Cref{sec:theory_smooth}, $n_\parallel\leq r$ restricts $\bmx$-prediction features to dominant tangent directions, while $n_\perp\leq d-r$ restricts $\bm\delta$-prediction features to the normal space, which they span when $n_\perp=d-r$. Because these dimensions are unknown, $n_\parallel$ and $n_\perp$ are treated as branch-capacity hyperparameters.

Each feature subset is provided exclusively to its corresponding decoder: the parallel decoder $\bmD^\parallel$ is trained for $\bmx$-prediction, while the orthogonal decoder $\bmD^\perp$ is trained for $\bm\delta$-prediction:
\begin{align}
    \cL_{\ae} \coloneqq \lambda_{\bmx} \cdot \bbE_{\bmx,\bm{\delta}}[\|\bmx - \bmD^\parallel(\bmphi^\parallel(\bmx + \bm{\delta})) \|_2^2] + \lambda_{\bm{\delta}} \cdot \bbE_{\bmx,\bm{\delta}}[\|\bm{\delta} - \bmD^\perp(\bmphi^\perp(\bmx + \bm{\delta})) \|_2^2]
\end{align}
The partition prevents the two decoders from sharing task-specific features while retaining a common feature extractor. By the analysis of \Cref{sec:theory_linear}, $\bmx$-prediction renders $\bmphi^\parallel$ insensitive to off-manifold corruption, whereas $\bm\delta$-prediction drives $\bmphi^\perp$ to capture deviations from the data manifold. These feature subsets enter the generative loss in \eqref{eq:decomposed_loss} directly.

\paragraph{Training the generator.}
Our method is two-stage: after training, the encoders are frozen and used to train the generator. Given $\bm{\varepsilon}$ from the prior, the generator produces $\hat{\bmx}_{\theta}\coloneqq\bmG_{\theta}(\bm{\varepsilon})$; the underlying method constructs a regression target $\bmx_{\targ}$, and we modify only the loss comparing these quantities.

For OT matching, $\bmx_{\targ}$ is assigned by minibatch OT between current outputs and real samples, following explicit Wasserstein minimization~\citep{fatras2020learning,chen2019gradual}. Given batches $\{\bm{\varepsilon}_{i}\}$ and $\{\bmx_{j}\}$, with $\hat{\bmx}_{\theta,i}=\bmG_{\theta}(\bm{\varepsilon}_{i})$, we compute
\begin{align}
    \textstyle \pi^{\star}\coloneqq\argmin_{\pi\in\Pi(B)}
    \sum_{i,j}\pi_{ij}\|\hat{\bmx}_{\theta,i}-\bmx_{j}\|_{2}^{2}.
    \label{eq:ot_coupling}
\end{align}
We then set $\bmx_{\targ}=\bmx_{j}$ for the data sample assigned to $\hat{\bmx}_{\theta,i}$, treating the assignment as fixed during the generator update. For a flow-map semigroup objective such as Riemannian MeanFlow~\citep{woo2026riemannianmeanflow}, let $\bmG_{\theta}^{s\rightarrow t}(\bmx_{s})$ denote the model prediction obtained by transporting $\bmx_{s}$ directly from time $s$ to $t$. The regression target is constructed by composing predictions over two shorter intervals:
\begin{align}
    \hat{\bmx}_{\theta} = \bmG_{\theta}^{s\rightarrow t}(\bmx_{s}), \quad
    \bmx_{\targ} = \sg\left[
        \bmG_{\theta}^{r\rightarrow t}
        \left(
            \bmG_{\theta}^{s\rightarrow r}(\bmx_{s})
        \right)
    \right].
    \label{eq:meanflow_target}
\end{align}
Here, $s\leq r\leq t$, and $\sg$ denotes stop-gradient. Additional terms required by the underlying method, such as the MeanFlow boundary loss, are retained.
Given $\bmx_{\mathrm{targ}}$, we train the generator using
\begin{align}
    \textstyle
    \cL_{\mathrm{gen}}(\theta) \coloneqq
    \bbE[
    \lambda_{\iso} \|\hat{\bmx}_\bmtheta - \bmx_{\targ}\|_2^2 +
    \lambda^{\parallel} \cL_{\bmphi^{\parallel}}(\hat{\bmx}_{\theta},\bmx_{\mathrm{targ}}) +
        \lambda^{\perp} \cL_{\bmphi^{\perp}}(\hat{\bmx}_{\theta}, \bmx_{\targ})].
    \label{eq:gen_loss}
\end{align}
Thus, the proposed metric is independent of how the regression target is constructed.
The squared discrepancy in either ambient or feature space may be replaced by any differentiable isotropic loss; the gradient-decomposition analysis extends directly, up to a loss-dependent scalar weighting of each residual.
\Cref{alg:manifold_decomposed_training} summarizes the complete two-stage procedure.

\begin{figure}[t]
\centering
\includegraphics[width=1.0\linewidth]{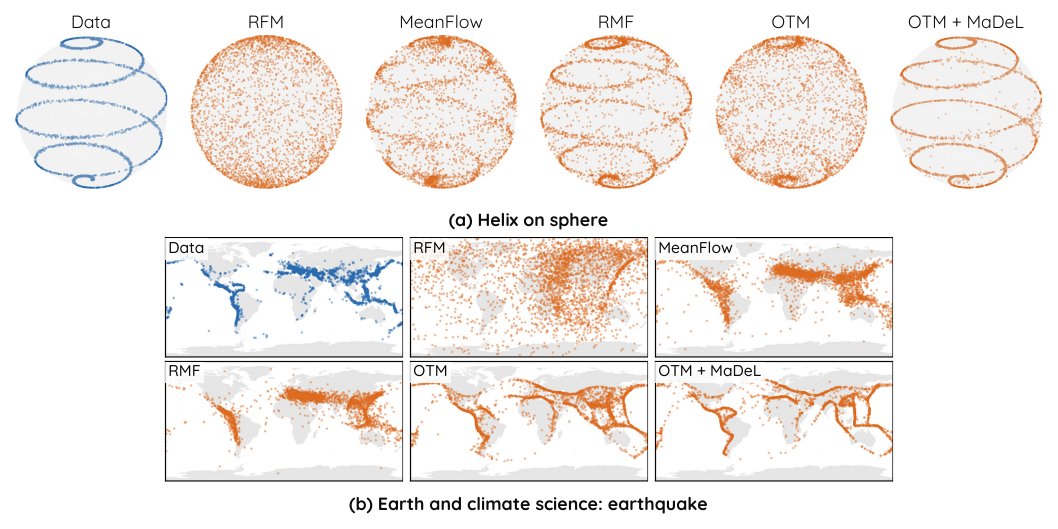}
\vspace{-8mm}
\caption{\textbf{Single-step generation on manifold-valued benchmarks.}}
\label{fig:qualitative_results}
\end{figure}

\vspace{-2mm}
\section{Experiments} \label{sec:experiments}

We evaluate \method{} on four benchmark families spanning synthetic and structured real-world distributions. Full experimental details are provided in \Cref{app:details}.

\textbf{Benchmarks.}
\emph{Helix on sphere}: a helix in $\mathbb{S}^{2}$, embedded isometrically into $\mathbb{S}^{512}\subset\mathbb{R}^{513}$ by a fixed orthonormal map and perturbed by tangential Gaussian noise.
\emph{Earth and climate science}: geographic coordinates of volcanic eruptions, earthquakes, floods, and wildfires represented on $\mathbb{S}^{2}$, following prior benchmarks for manifold-valued generation~\citep{debortoli2022riemannian,chen2024flow}.
\emph{Protein and RNA torsion angles}: the preprocessed benchmarks of Riemannian Flow Matching~\citep{chen2024flow}, following \citet{huang2022riemannian}, comprising four Ramachandran-angle distributions (General, Glycine, Proline, and Pre-Pro) on $\mathbb{T}^{2}$ and a seven-dimensional RNA torsion distribution on $\mathbb{T}^{7}$.
\emph{Protein backbones}: SCOPe protein structures~\citep{chandonia2022scope}, following the preprocessing and evaluation protocol of FrameFlow~\citep{yim2023fast}.

\begin{table}[t]
\centering
\resizebox{1.0\columnwidth}{!}{
\begin{tabular}{lccccccccccc}
\toprule
\multirow{2}{*}{\textbf{Method}} & \multicolumn{5}{c}{\textbf{Earth and Climate Science Datasets} ($W_1\downarrow$)} & \multicolumn{6}{c}{\textbf{Protein and RNA Torsion-Angle Datasets} ($W_1\downarrow$)} \\
\cmidrule(lr){2-6} \cmidrule(lr){7-12}
 & Earthquake & Fire & Flood & Volcano & Avg. & General & Glycine & Proline & Pre-Pro & RNA & Avg. \\
\midrule
RFM    & $27.19^\circ$ & $33.04^\circ$ & $23.68^\circ$         & $28.40^\circ$ & $28.08^\circ$
& $57.79^\circ$ & $64.80^\circ$ & $60.74^\circ$ & $32.68^\circ$ & $63.55^\circ$ & $55.912^\circ$ \\
MF     & $6.43^\circ$  & $9.22^\circ$  & $6.67^\circ$          & $14.95^\circ$ & $9.32^\circ$
& $17.27^\circ$ & $11.57^\circ$ & $16.25^\circ$ & $13.77^\circ$ & $35.62^\circ$ & $18.896^\circ$ \\
RMF    & $9.91^\circ$  & $10.21^\circ$ & $8.39^\circ$          & $16.82^\circ$ & $11.33^\circ$
& $34.33^\circ$ & $24.63^\circ$ & $31.84^\circ$ & $15.21^\circ$ & $41.84^\circ$ & $29.570^\circ$ \\
OTM    & $6.22^\circ$  & $6.25^\circ$  & $\mathbf{4.88}^\circ$ & $6.22^\circ$  & $5.89^\circ$
& $4.81^\circ$ & $8.27^\circ$ & $4.97^\circ$ & $\mathbf{5.97}^\circ$ & $23.68^\circ$ & $9.540^\circ$ \\
\, + \method{} & $\mathbf{5.81}^\circ$ & $\mathbf{5.93}^\circ$ & $5.16^\circ$ & $\mathbf{5.46}^\circ$ & $\mathbf{5.59}^\circ$
& $\mathbf{4.51}^\circ$ & $\mathbf{6.69}^\circ$ & $\mathbf{4.39}^\circ$ & $6.03^\circ$ & $\mathbf{22.52}^\circ$ & $\mathbf{8.828}^\circ$ \\
\bottomrule 
\end{tabular}}
\caption{\textbf{Single-step generation error on Earth and torsion-angle datasets.} Angular $W_1$ distance (degrees) between generated and ground-truth distributions; lower is better. RFM: Riemannian Flow Matching; MF, RMF: (Riemannian) MeanFlow; OTM: coordinate-space OT matching; + \method{}: OTM with our loss. All methods use single-step sampling.}
\label{tab:error}
\vspace{-2mm}
\end{table}

\textbf{Baselines.}
We compare with Riemannian Flow Matching (RFM)~\citep{chen2024flow}, MeanFlow (MF)~\citep{geng2025meanflow}, and Riemannian MeanFlow (RMF)~\citep{woo2026riemannianmeanflow}, each trained with its native objective. We also compare with coordinate-space OT matching (OTM), which computes a minibatch OT coupling between generated samples and real samples and regresses each generated output toward its assigned real-data target. OTM and OTM + \method{} use the same coupling, generator architecture, and sampling procedure, differing only in the regression loss.

\vspace{-2mm}
\subsection{Main Results} \label{sec:results}
\vspace{-1mm}

\textbf{Helix on sphere.}
In \Cref{fig:qualitative_results} top, RFM, MF, RMF, and OTM distribute samples broadly over the sphere, failing to concentrate on the helical support. The Riemannian methods preserve the geometry of the spherical state space but do not recover the lower-dimensional data manifold embedded within it. In contrast, OTM + \method{} reconstructs the global helical structure and places more probability mass near the data support using the same single network evaluation.

\textbf{Earth and climate science.}
As shown in \Cref{fig:qualitative_results} bottom, RFM spreads samples broadly across the surface, while MF and RMF recover the major high-density regions but blur their finer structure. 
OTM captures much of the global support but retains noticeable dispersion away from the narrow data structures.
Using the same OT assignments, augmenting the coordinate-space regression loss with \method{} produces more sharply localized samples while preserving broad geographic coverage.
In \Cref{tab:error} left, OTM + \method{} achieves the lowest angular distance on three of the four datasets and improves the Earth average from $5.89^\circ$ to $5.59^\circ$; OTM remains strongest on Flood.

\textbf{Protein and RNA torsion angles.}
In \Cref{tab:error} right, OTM + \method{} achieves the lowest angular $W_1$ distance on four of the five datasets and improves the mean distance from $9.540^\circ$ to $8.828^\circ$ over OTM. It remains essentially tied on Pre-Pro ($6.03^\circ$ versus $5.97^\circ$), with the largest improvement on Glycine, the broadest Ramachandran distribution.

\paragraph{Protein backbone generation.}
We train identical RMF models, with \method{} adding a frozen decomposed-feature term to the standard semigroup objective.
Although \method{} does not improve designability, it improves local steric geometry. We measure the maximum van der Waals overlap between backbone $\mathrm{N}$, $\mathrm{C}_{\alpha}$, and $\mathrm{C}$ atoms of residues separated by at least two positions; lower values indicate fewer steric clashes, and real SCOPe backbones score $0.28$--$0.30$~\AA. As shown in \Cref{fig:scope}, across eight paired prior batches from the one-million-step checkpoints, \method{} reduces the median maximum overlap at all ten tested lengths and all three sampling budgets. At $L=128$, it reduces overlap from $1.375$ to $1.011$~\AA{} at $\mathrm{NFE}=10$ (lower in $8/8$ batches), from $1.742$ to $1.503$~\AA{} at $\mathrm{NFE}=5$ ($7/8$), and from $2.732$ to $2.681$~\AA{} at $\mathrm{NFE}=1$ ($6/8$). This is not explained by chain expansion: at $L=60$, \method{} produces backbones that are both more compact and less clashing in $6$--$7$ of eight batches. Designability remains essentially unchanged ($10/640$ versus $12/640$), as does median self-consistency RMSD ($8.63$ versus $8.66$~\AA{} at $L=128$ and $\mathrm{NFE}=10$).


\begin{figure}[t]
\centering
\includegraphics[width=1.0\linewidth]{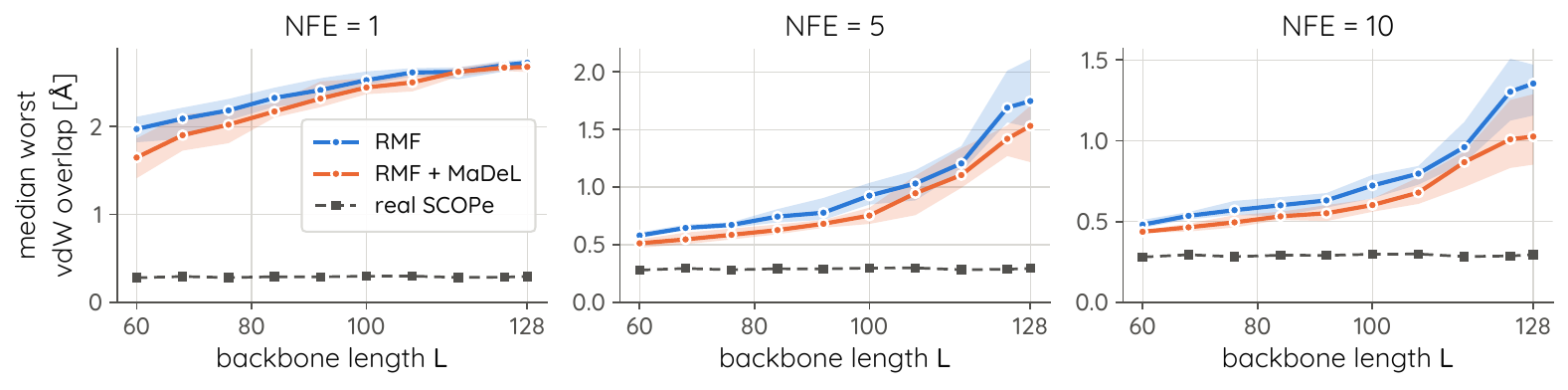}
\vspace{-4mm}
\caption{\textbf{Steric violation versus backbone length.} Deepest backbone van der Waals overlap per structure (N/C$\alpha$/C, residues $\geq 2$ apart), median over 8 independent batches of 128 paired samples from each model's 1M-step checkpoint; bands are the full batch range, dashed is real SCOPe. \method{} is lower at every length and budget.}
\vspace{-4mm}
\label{fig:scope}
\end{figure}

\subsection{Ablation Study and Analysis} \label{sec:ablation}

\begin{wraptable}{r}{0.5\linewidth}
\begin{minipage}{0.5\textwidth}
\vspace{-4mm}
\centering
\resizebox{1.0\textwidth}{!}{
\begin{tabular}{lccccc}
\toprule
\textbf{Loss} & Earthquake & Fire & Flood & Volcano & Avg. \\
\midrule
OTM & $6.22^\circ$ & $6.25^\circ$ & $\mathbf{4.88}^\circ$ & $6.22^\circ$ & $5.89^\circ$ \\
+ Overparam. AE & $5.89^\circ$ & $6.16^\circ$ & $5.08^\circ$ & $6.24^\circ$ & $5.84^\circ$ \\
+ $\parallel$ only ($\lambda^\perp=0$) & $6.12^\circ$ & $6.94^\circ$ & $5.41^\circ$ & $6.04^\circ$ & $6.13^\circ$ \\
+ $\perp$ only ($\lambda^\parallel=0$) & $6.59^\circ$ & $6.42^\circ$ & $5.17^\circ$ & $6.15^\circ$ & $6.08^\circ$ \\
+ \method{} & $\mathbf{5.81}^\circ$ & $\mathbf{5.93}^\circ$ & $5.16^\circ$ & $\mathbf{5.46}^\circ$ & $\mathbf{5.59}^\circ$ \\
\bottomrule
\end{tabular}}
\caption{\textbf{Ablation of the decomposed loss.}
Angular $W_1$ under the same OT pairing.}
\label{tab:ablation}
\vspace{-2mm}
\end{minipage}
\end{wraptable}

\paragraph{Ablation.}
We ablate the two feature branches under the same OT pairing (\Cref{tab:ablation}). While Flood is an exception, with all feature-based variants underperforming coordinate-space OTM, the average results show a clear trend: the overparameterized autoencoder yields only a modest improvement from $5.89^\circ$ to $5.84^\circ$, either manifold-decomposed branch alone performs worse ($6.13^\circ$ for $\parallel$ and $6.08^\circ$ for $\perp$), and combining both branches achieves the best average of $5.59^\circ$. \method{} also attains the lowest error on Earthquake, Fire, and Volcano, indicating that the complementary branches provide the strongest overall performance despite variation across datasets.

\begin{wraptable}{r}{0.33\linewidth}
\begin{minipage}{0.33\textwidth}
\vspace{-4mm}
\centering
\resizebox{1.0\textwidth}{!}{
\begin{tabular}{lcccc}
\toprule
& $\sigma = 0.0$ & $0.01$ & $0.05$ & $0.1$ \\
\midrule
$\bmphi^\parallel$ & $15.88$ & $16.75$ & $45.04$ & $33.21$ \\
$\bmphi^\perp$ & $0.02$ & $0.02$ & $0.04$ & $0.03$ \\
\bottomrule
\end{tabular}}
\caption{\textbf{Jacobian tangent-to-normal energy ratio.}}
\label{tab:energy}
\vspace{-4mm}
\end{minipage}
\end{wraptable}

\textbf{Analysis.}
To test whether the encoders realize the decomposition predicted by \Cref{sec:theory_smooth}, we use the tangent and normal spaces of the helix and measure the ratio of manifold-parallel to manifold-orthogonal Jacobian energy (\Cref{tab:energy}).
The ratio is evaluated at samples perturbed by noise with standard deviation $\sigma$. Across noise levels, $\bmphi^\parallel$ produces large ratios ($15.88$--$45.04$), whereas $\bmphi^\perp$ produces ratios near zero ($0.02$--$0.04$), confirming that energies concentrate in complementary directions on and near the manifold.

\vspace{-2mm}
\section{Conclusion} \label{sec:discussion}
\vspace{-2mm}

We introduced \method{}, a manifold-decomposed feature loss for generative modeling that distinguishes variation along the data manifold from deviation away from it.
We showed that, under a feature bottleneck, $\bmx$-prediction and $\bm\delta$-prediction learn encoders whose Jacobians are confined to the tangent and normal spaces of the manifold, exactly for linear manifolds and locally for smooth manifolds, so that the resulting feature loss yields complementary tangent- and normal-aligned generative updates without explicit access to the manifold.
Across controlled manifolds and structured real-world domains, \method{} better preserves data support and improves distributional or local geometric accuracy relative to isotropic regression, which suggests that explicitly shaping the geometry of the training signal is a practical alternative to parameterizing the data manifold itself, particularly in sparse and highly constrained generative problems.

\paragraph{Limitations.}
Our analysis is local and holds to leading order in the corruption level, while the branch dimensions and loss weights must be chosen without knowledge of the intrinsic dimension. Our experiments cover a limited set of synthetic and structured scientific benchmarks, and do not yet establish scalability across models, datasets, or modalities. Future work will develop \method{} for scalable manifold generative modeling; \Cref{app:limitations} discusses these limitations.


\bibliography{iclr2027_conference}

@InProceedings{johnson2016perceptual,
author={Justin Johnson and Alexandre Alahi and Fei-Fei Li},
title={{Perceptual Losses for Real-Time Style Transfer and Super-Resolution}},
booktitle={ECCV},
year={2016}
}

@InProceedings{zhang2018unreasonable,
author={Richard Zhang and Phillip Isola and Alexei A. Efros and Eli Shechtman and Oliver Wang},
title={{The Unreasonable Effectiveness of Deep Features as a Perceptual Metric}},
booktitle={CVPR},
year={2018}
}

@InProceedings{sauer2021projected,
author={Axel Sauer and Kashyap Chitta and Jens M{\"u}ller and Andreas Geiger},
title={{Projected GANs Converge Faster}},
booktitle={NeurIPS},
year={2021}
}

@InProceedings{song2023consistency,
author={Yang Song and Prafulla Dhariwal and Mark Chen and Ilya Sutskever},
title={{Consistency Models}},
booktitle={ICML},
year={2023}
}

@InProceedings{huang2022riemannian,
author={Chin-Wei Huang and Milad Aghajohari and Avishek Joey Bose and Prakash Panangaden and Aaron Courville},
title={{Riemannian Diffusion Models}},
booktitle={NeurIPS},
year={2022}
}

@InProceedings{chen2024flow,
author={Ricky T. Q. Chen and Yaron Lipman},
title={{Flow Matching on General Geometries}},
booktitle={ICLR},
year={2024}
}

@InProceedings{stanczuk2024diffusion,
author={Jan Pawel Stanczuk and Georgios Batzolis and Teo Deveney and Carola-Bibiane Sch{\"o}nlieb},
title={{Diffusion Models Encode the Intrinsic Dimension of Data Manifolds}},
booktitle={ICML},
year={2024}
}

@InProceedings{sun2023sddm,
author={Shikun Sun and Longhui Wei and Junliang Xing and Jia Jia and Qi Tian},
title={{SDDM: Score-Decomposed Diffusion Models on Manifolds for Unpaired Image-to-Image Translation}},
booktitle={ICML},
year={2023}
}

@InProceedings{ho2020ddpm,
author={Jonathan Ho and Ajay Jain and Pieter Abbeel},
title={{Denoising Diffusion Probabilistic Models}},
booktitle={NeurIPS},
year={2020}
}

@InProceedings{lipman2023flow,
author={Yaron Lipman and Ricky T. Q. Chen and Heli Ben-Hamu and Maximilian Nickel and Matthew Le},
title={{Flow Matching for Generative Modeling}},
booktitle={ICLR},
year={2023}
}

@article{fefferman2016manifold,
author={Charles Fefferman and Sanjoy Mitter and Hariharan Narayanan},
title={{Testing the Manifold Hypothesis}},
journal={Journal of the American Mathematical Society},
volume={29},
number={4},
pages={983--1049},
year={2016}
}

@InProceedings{yim2023se3,
author={Jason Yim and Brian L. Trippe and Valentin De Bortoli and Emile Mathieu and Arnaud Doucet and Regina Barzilay and Tommi Jaakkola},
title={{SE(3) Diffusion Model with Application to Protein Backbone Generation}},
booktitle={ICML},
year={2023}
}

@InProceedings{woo2026riemannianmeanflow,
author={Dongyeop Woo and Marta Skreta and Seonghyun Park and Kirill Neklyudov and Sungsoo Ahn},
title={{Riemannian MeanFlow}},
booktitle={ICML},
year={2026}
}

@InProceedings{song2021score,
author={Yang Song and Jascha Sohl-Dickstein and Diederik P. Kingma and Abhishek Kumar and Stefano Ermon and Ben Poole},
title={{Score-Based Generative Modeling through Stochastic Differential Equations}},
booktitle={ICLR},
year={2021}
}

@InProceedings{liu2023rectified,
author={Xingchao Liu and Chengyue Gong and Qiang Liu},
title={{Flow Straight and Fast: Learning to Generate and Transfer Data with Rectified Flow}},
booktitle={ICLR},
year={2023}
}

@InProceedings{karras2022edm,
author={Tero Karras and Miika Aittala and Timo Aila and Samuli Laine},
title={{Elucidating the Design Space of Diffusion-Based Generative Models}},
booktitle={NeurIPS},
year={2022}
}

@InProceedings{ho2022video,
author={Jonathan Ho and Tim Salimans and Alexey Gritsenko and William Chan and Mohammad Norouzi and David J. Fleet},
title={{Video Diffusion Models}},
booktitle={NeurIPS},
year={2022}
}

@InProceedings{hoogeboom2022equivariant,
author={Emiel Hoogeboom and V{\'i}ctor Garcia Satorras and Cl{\'e}ment Vignac and Max Welling},
title={{Equivariant Diffusion for Molecule Generation in 3D}},
booktitle={ICML},
year={2022}
}

@Article{tenenbaum2000isomap,
author={Joshua B. Tenenbaum and Vin de Silva and John C. Langford},
title={{A Global Geometric Framework for Nonlinear Dimensionality Reduction}},
journal={Science},
year={2000}
}

@InProceedings{narayanan2010manifold,
author={Hariharan Narayanan and Sanjoy Mitter},
title={{Sample Complexity of Testing the Manifold Hypothesis}},
booktitle={NeurIPS},
year={2010}
}

@InProceedings{pope2021intrinsic,
author={Phillip Pope and Chen Zhu and Ahmed Abdelkader and Micah Goldblum and Tom Goldstein},
title={{The Intrinsic Dimension of Images and Its Impact on Learning}},
booktitle={ICLR},
year={2021}
}

@InProceedings{rifai2011manifold,
author={Salah Rifai and Yann N. Dauphin and Pascal Vincent and Yoshua Bengio and Xavier Muller},
title={{The Manifold Tangent Classifier}},
booktitle={NeurIPS},
year={2011}
}

@Article{alain2014regularized,
author={Guillaume Alain and Yoshua Bengio},
title={{What Regularized Auto-Encoders Learn from the Data-Generating Distribution}},
journal={JMLR},
year={2014}
}

@InProceedings{debortoli2022riemannian,
author={Valentin De Bortoli and Emile Mathieu and Michael Hutchinson and James Thornton and Yee Whye Teh and Arnaud Doucet},
title={{Riemannian Score-Based Generative Modelling}},
booktitle={NeurIPS},
year={2022}
}

@Article{chandonia2022scope,
author={John-Marc Chandonia and Lindsey Guan and Shiangyi Lin and Changhua Yu and Naomi K. Fox and Steven E. Brenner},
title={{SCOPe: Improvements to the Structural Classification of Proteins---Extended Database to Facilitate Variant Interpretation and Machine Learning}},
journal={Nucleic Acids Research},
volume={50},
number={D1},
pages={D553--D559},
year={2022}
}

@Article{yim2023fast,
author={Jason Yim and Andrew Campbell and Andrew Y. K. Foong and Michael Gastegger and Jos{\'e} Jim{\'e}nez-Luna and Sarah Lewis and Victor Garcia Satorras and Bastiaan S. Veeling and Regina Barzilay and Tommi Jaakkola and Frank No{\'e}},
title={{Fast Protein Backbone Generation with SE(3) Flow Matching}},
journal={arXiv preprint arXiv:2310.05297},
year={2023}
}

@InProceedings{kim2026ayt,
author={Beomsu Kim and ByungHee Cha and Jong Chul Ye},
title={{Align Your Trajectory Tangent: Training Better Consistency Models via Manifold-Aligned Tangents}},
booktitle={ICML},
year={2026}
}

@InProceedings{song2024icm,
author={Yang Song and Prafulla Dhariwal},
title={{Improved Techniques for Training Consistency Models}},
booktitle={ICLR},
year={2024}
}

@InProceedings{he2024mpgd,
author={Yutong He and Naoki Murata and Chieh-Hsin Lai and Yuhta Takida and Toshimitsu Uesaka and Dongjun Kim and Wei-Hsiang Liao and Yuki Mitsufuji and J. Zico Kolter and Ruslan Salakhutdinov and Stefano Ermon},
title={{Manifold Preserving Guided Diffusion}},
booktitle={ICLR},
year={2024}
}

@InProceedings{geng2025meanflow,
author={Zhengyang Geng and Mingyang Deng and Xingjian Bai and J. Zico Kolter and Kaiming He},
title={{Mean Flows for One-Step Generative Modeling}},
booktitle={NeurIPS},
year={2025}
}

@InProceedings{vincent2008extracting,
author={Pascal Vincent and Hugo Larochelle and Yoshua Bengio and Pierre-Antoine Manzagol},
title={{Extracting and Composing Robust Features with Denoising Autoencoders}},
booktitle={ICML},
year={2008}
}

@InProceedings{rifai2011contractive,
author={Salah Rifai and Pascal Vincent and Xavier Muller and Xavier Glorot and Yoshua Bengio},
title={{Contractive Auto-Encoders: Explicit Invariance During Feature Extraction}},
booktitle={ICML},
year={2011}
}

@InProceedings{liu2025tango,
author={Zichen Liu and Wei Zhang and Tiejun Li},
title={{Improving the Euclidean Diffusion Generation of Manifold Data by Mitigating Score Function Singularity}},
booktitle={NeurIPS},
year={2025}
}

@InProceedings{saito2025tangential,
author={Shinnosuke Saito and Takashi Matsubara},
title={{Be Tangential to Manifold: Discovering Riemannian Metric for Diffusion Models}},
booktitle={ECCV},
year={2026}
}

@Article{dauparas2022robust,
author={Justas Dauparas and Ivan Anishchenko and Nathaniel Bennett and Hua Bai and Robert J. Ragotte and Lukas F. Milles and Basile I. M. Wicky and Alexis Courbet and Rob J. de Haas and Neville Bethel and others},
title={{Robust Deep Learning-Based Protein Sequence Design Using ProteinMPNN}},
journal={Science},
volume={378},
number={6615},
pages={49--56},
year={2022}
}

@Article{lin2023evolutionary,
author={Zeming Lin and Halil Akin and Roshan Rao and Brian Hie and Zhongkai Zhu and Wenting Lu and Nikita Smetanin and Robert Verkuil and Ori Kabeli and Yaniv Shmueli and Allan dos Santos Costa and Maryam Fazel-Zarandi and Tom Sercu and Salvatore Candido and Alexander Rives},
title={{Evolutionary-Scale Prediction of Atomic-Level Protein Structure with a Language Model}},
journal={Science},
volume={379},
number={6637},
pages={1123--1130},
year={2023}
}

@Article{bondi1964van,
author={Abraham Bondi},
title={{Van der Waals Volumes and Radii}},
journal={The Journal of Physical Chemistry},
volume={68},
number={3},
pages={441--451},
year={1964}
}

@InProceedings{chen2019gradual,
author={Yucheng Chen and Matus Telgarsky and Chao Zhang and Bolton Bailey and Daniel Hsu and Jian Peng},
title={{A Gradual, Semi-Discrete Approach to Generative Network Training via Explicit Wasserstein Minimization}},
booktitle={ICML},
year={2019}
}

@InProceedings{fatras2020learning,
author={Kilian Fatras and Younes Zine and Rémi Flamary and Rémi Gribonval and Nicolas Courty},
title={{Learning with Minibatch Wasserstein: Asymptotic and Gradient Properties}},
booktitle={AISTATS},
year={2020}
}
\bibliographystyle{iclr2027_conference}

\appendix

\section{Proofs}

\subsection{Linear manifolds}

\begin{proof}[Proof of Proposition~\ref{prop:linear_parallel}]
Let $\bmy\coloneqq\bmx+\bm{\delta}$ and $\bmC\coloneqq\bbE[\bmy\bmy^\top]=\bm{\Sigma}+\sigma^2\bmI$. Since $\sigma>0$, $\bmC$ is positive definite. The objective depends on the encoder and decoder only through $\bmF\coloneqq\bmD\bmE$, where $\rank(\bmF)\leq n_\parallel$. Using $\bbE[\bmx\bmy^\top]=\bm{\Sigma}$ gives
\begin{align}
\bbE\|\bmx-\bmF\bmy\|_2^2
&=\operatorname{tr}(\bm{\Sigma})-2\operatorname{tr}(\bmF\bm{\Sigma})
  +\operatorname{tr}(\bmF\bmC\bmF^\top) \\
&=\|\bmF\bmC^{1/2}-\bm{\Sigma}\bmC^{-1/2}\|_F^2
  +\operatorname{tr}(\bm{\Sigma})-\operatorname{tr}(\bm{\Sigma}\bmC^{-1}\bm{\Sigma}).
\end{align}
Hence, after defining $\bmG\coloneqq\bmF\bmC^{1/2}$, minimizing the objective is equivalent to
\begin{align}
\min_{\rank(\bmG)\leq n_\parallel}
\|\bmG-\bm{\Sigma}\bmC^{-1/2}\|_F^2.
\end{align}
Because $\bmC=\bmU\diag(\lambda_i+\sigma^2)\bmU^\top$ commutes with $\bm{\Sigma}$,
\begin{align}
\bm{\Sigma}\bmC^{-1/2}
=\bmU\diag\!\left(
\frac{\lambda_1}{\sqrt{\lambda_1+\sigma^2}},\ldots,
\frac{\lambda_d}{\sqrt{\lambda_d+\sigma^2}}
\right)\bmU^\top.
\end{align}
The scalar map $\lambda\mapsto\lambda/\sqrt{\lambda+\sigma^2}$ is strictly increasing on $[0,\infty)$. The Eckart--Young--Mirsky theorem and the assumed eigengap therefore show that the unique optimal rank-$n_\parallel$ approximation retains the first $n_\parallel$ eigendirections:
\begin{align}
\bmG^\star
=\bmU_{n_\parallel}
\diag\!\left(
\frac{\lambda_1}{\sqrt{\lambda_1+\sigma^2}},\ldots,
\frac{\lambda_{n_\parallel}}{\sqrt{\lambda_{n_\parallel}+\sigma^2}}
\right)
\bmU_{n_\parallel}^\top.
\end{align}
Multiplying by $\bmC^{-1/2}$ on the right yields
\begin{align}
\bmF^\star
=\bmU_{n_\parallel}\bm{\Gamma}_\parallel
 \bmU_{n_\parallel}^\top.
\end{align}
Since $n_\parallel\rev{\leq} r$, $\bm{\Gamma}_\parallel$ is invertible and $\rank(\bmF^\star)=n_\parallel$. Thus, any factorization $\bmF^\star=\bmD^\star\bmE^\star$ through an $n_\parallel$-dimensional bottleneck has $\rank(\bmD^\star)=\rank(\bmE^\star)=n_\parallel$ and $\Image(\bmD^\star)=\Span(\bmU_{n_\parallel})$. Write $\bmD^\star=\bmU_{n_\parallel}\bmA$ for an invertible $\bmA$. Then $\bmE^\star=\bmA^{-1}\bm{\Gamma}_\parallel\bmU_{n_\parallel}^\top$. Setting $\bmR\coloneqq\bmA^{-1}\bm{\Gamma}_\parallel$ gives
\begin{align}
\bmE^\star=\bmR\bmU_{n_\parallel}^\top,
\qquad
\bmD^\star=\bmU_{n_\parallel}\bm{\Gamma}_\parallel\bmR^{-1},
\end{align}
which characterizes every optimal factorization.
\end{proof}

\begin{proof}[Proof of Proposition~\ref{prop:linear_orthogonal}]
Again let $\bmy=\bmx+\bm{\delta}$, $\bmC=\bm{\Sigma}+\sigma^2\bmI$, and $\bmF=\bmD\bmE$. Since $\bbE[\bm{\delta}\bmy^\top]=\sigma^2\bmI$,
\begin{align}
\bbE\|\bm{\delta}-\bmF\bmy\|_2^2
&=d\sigma^2-2\sigma^2\operatorname{tr}(\bmF)
  +\operatorname{tr}(\bmF\bmC\bmF^\top) \\
&=\|\bmF\bmC^{1/2}-\sigma^2\bmC^{-1/2}\|_F^2
  +d\sigma^2-\sigma^4\operatorname{tr}(\bmC^{-1}).
\end{align}
With $\bmG\coloneqq\bmF\bmC^{1/2}$, it remains to find the best rank-$n_\perp$ approximation of
\begin{align}
\sigma^2\bmC^{-1/2}
=\bmU\diag\!\left(
\frac{\sigma^2}{\sqrt{\lambda_1+\sigma^2}},\ldots,
\frac{\sigma^2}{\sqrt{\lambda_d+\sigma^2}}
\right)\bmU^\top.
\end{align}
These singular values decrease as $\lambda_i$ increases. Thus, the Eckart--Young--Mirsky theorem selects the $n_\perp$ smallest-variance eigendirections, giving
\begin{align}
\bmG^\star
=\bmU_{-n_\perp}
\diag\!\left(
\frac{\sigma^2}{\sqrt{\lambda_{d-n_\perp+1}+\sigma^2}},\ldots,
\frac{\sigma^2}{\sqrt{\lambda_d+\sigma^2}}
\right)
\bmU_{-n_\perp}^\top.
\end{align}
Therefore,
\begin{align}
\bmF^\star=\bmG^\star\bmC^{-1/2}
=\bmU_{-n_\perp}\bm{\Gamma}_\perp\bmU_{-n_\perp}^\top.
\end{align}
The diagonal entries of $\bm{\Gamma}_\perp$ are strictly positive, so $\rank(\bmF^\star)=n_\perp$. The same factorization argument as in Proposition~\ref{prop:linear_parallel} then yields
\begin{align}
\bmE^\star=\bmR\bmU_{-n_\perp}^\top,
\qquad
\bmD^\star=\bmU_{-n_\perp}\bm{\Gamma}_\perp\bmR^{-1}
\end{align}
for an arbitrary invertible $\bmR$. If eigenvalues are tied at the truncation boundary, $\bmU_{-n_\perp}$ may be any orthonormal basis of an optimal subspace within the tied eigenspace.
\end{proof}

\subsection{General smooth manifolds}

\begin{proof}[Proof of Proposition~\ref{prop:general_parallel}]
Define the local autoencoder map $\bmA_{\ae}\coloneqq\bmD\circ\bmE$. The manifold parameterization has the first-order expansion
\begin{align}
\bmx_\epsilon
=\bmx_0+\epsilon\bmJ_\iota\bm{\xi}+\bmr_\epsilon,
\qquad
\bmr_\epsilon=o(\epsilon)
\end{align}
in mean square. Since $\bmA_{\ae}$ is twice continuously differentiable and $\bmA_{\ae}(\bmx_0)=\bmx_0$, the chain rule gives
\begin{align}
\bmA_{\ae}(\bmx_\epsilon+\bm{\delta}_\sigma)
=\bmx_0+\bmD_1\bmE_1
 (\epsilon\bmJ_\iota\bm{\xi}+\bm{\delta}_\sigma)
+\bmr_{\ae},
\end{align}
where the combined Taylor remainder satisfies
$\bbE\|\bmr_\epsilon-\bmr_{\ae}\|_2^2=o(\epsilon^2+\sigma^2)$ and its cross term with the first-order residual is also $o(\epsilon^2+\sigma^2)$. Consequently,
\begin{align}
&\bbE_{\bm{\xi},\bm{\delta}_\sigma}
\|\bmx_\epsilon-\bmD(\bmE(\bmx_\epsilon+\bm{\delta}_\sigma))\|_2^2 \\
&\quad=
\bbE_{\bm{\xi},\bm{\delta}_\sigma}
\|\epsilon\bmJ_\iota\bm{\xi}
-\bmD_1\bmE_1(\epsilon\bmJ_\iota\bm{\xi}+\bm{\delta}_\sigma)\|_2^2
+o(\epsilon^2+\sigma^2).
\end{align}
The leading-order problem is precisely the reduced-rank denoising problem of Proposition~\ref{prop:linear_parallel}, with clean covariance $\epsilon^2\bmC_T$ and corruption covariance $\sigma^2\bmI$. Its eigendirections are those of $\bmC_T$, and the corresponding gains are
\begin{align}
\frac{\epsilon^2\lambda_i}{\epsilon^2\lambda_i+\sigma^2}.
\end{align}
Applying Proposition~\ref{prop:linear_parallel} therefore gives the stated $\bmD_1^\star\bmE_1^\star$, $\bmE_1^\star$, and $\bmD_1^\star$.
\end{proof}

\begin{proof}[Proof of Proposition~\ref{prop:general_orthogonal}]
Use the same expansion
$\bmx_\epsilon=\bmx_0+\epsilon\bmJ_\iota\bm{\xi}+\bmr_\epsilon$.
Now $\bmA_{\ae}(\bmx_0)=\bmD(\bmE(\bmx_0))=\bm{0}$, so Taylor expansion and the chain rule yield
\begin{align}
\bmD(\bmE(\bmx_\epsilon+\bm{\delta}_\sigma))
=\bmD_1\bmE_1
(\epsilon\bmJ_\iota\bm{\xi}+\bm{\delta}_\sigma)
+\bmr_{\ae}.
\end{align}
As in the preceding proof, the squared remainder and its cross term with the first-order residual are $o(\epsilon^2+\sigma^2)$. Hence,
\begin{align}
&\bbE_{\bm{\xi},\bm{\delta}_\sigma}
\|\bm{\delta}_\sigma-\bmD(\bmE(\bmx_\epsilon+\bm{\delta}_\sigma))\|_2^2 \\
&\quad=
\bbE_{\bm{\xi},\bm{\delta}_\sigma}
\|\bm{\delta}_\sigma
-\bmD_1\bmE_1(\epsilon\bmJ_\iota\bm{\xi}+\bm{\delta}_\sigma)\|_2^2
+o(\epsilon^2+\sigma^2).
\end{align}
This is the linear $\bm\delta$-prediction problem of Proposition~\ref{prop:linear_orthogonal} with data covariance $\epsilon^2\bmC_T$. The optimal rank-$n_\perp$ map therefore retains the $n_\perp$ smallest eigenvalues of $\bmC_T$ and assigns gain
\begin{align}
\frac{\sigma^2}{\epsilon^2\lambda_i+\sigma^2}
\end{align}
to the eigendirection with eigenvalue $\lambda_i$. Proposition~\ref{prop:linear_orthogonal} then gives the stated composite map and all of its optimal factorizations.
\end{proof}

\subsection{Directional decomposition of the generative loss}

\begin{proof}[Proof of Proposition~\ref{prop:generative_gradient_decomposition}]
By Proposition~\ref{prop:general_parallel}, the Jacobian of a locally optimal manifold-parallel encoder has the form
\begin{align}
\bmJ_{\bmphi^\parallel}(\bmx_0)
=\bmR_\parallel\bmU_{n_\parallel}^\top.
\end{align}
The columns of $\bmU_{n_\parallel}$ are positive-eigenvalue eigenvectors of $\bmC_T$ and therefore lie in $\Image(\bmC_T)=T_{\bmx_0}\cM$. Since $\bmR_\parallel$ is invertible,
\begin{align}
\Image((\bmJ_{\bmphi^\parallel}(\bmx_0))^\top)
=\Span(\bmU_{n_\parallel})
\subseteq T_{\bmx_0}\cM.
\end{align}
Similarly, Proposition~\ref{prop:general_orthogonal} gives
\begin{align}
\bmJ_{\bmphi^\perp}(\bmx_0)
=\bmR_\perp\bmU_{-n_\perp}^\top.
\end{align}
Because $n_\perp\leq d-r$, the columns of $\bmU_{-n_\perp}$ can be chosen entirely from the zero-eigenspace of $\bmC_T$, which equals $N_{\bmx_0}\cM$. Thus,
\begin{align}
\Image((\bmJ_{\bmphi^\perp}(\bmx_0))^\top)
=\Span(\bmU_{-n_\perp})
\subseteq N_{\bmx_0}\cM.
\end{align}
For a target $\bmx$, define the two local feature-gradient components at $\hat{\bmx}=\bmx_0$ by
\begin{align}
\bmg^\parallel
&\coloneqq
(\bmJ_{\bmphi^\parallel}(\bmx_0))^\top
\bigl(\bmphi^\parallel(\bmx_0)-\bmphi^\parallel(\bmx)\bigr), \\
\bmg^\perp
&\coloneqq
(\bmJ_{\bmphi^\perp}(\bmx_0))^\top
\bigl(\bmphi^\perp(\bmx_0)-\bmphi^\perp(\bmx)\bigr).
\end{align}
The two image-space inclusions above imply $\bmg^\parallel\in T_{\bmx_0}\cM$ and $\bmg^\perp\in N_{\bmx_0}\cM$. Substituting these definitions into the gradient identity in \eqref{eq:feature_gradient} yields
\begin{align}
\nabla_{\hat{\bmx}}\cL
=\lambda^\parallel\bmg^\parallel
+\lambda^\perp\bmg^\perp,
\end{align}
which proves the claimed directional decomposition.
\end{proof}

\section{Scope of the Analysis and Limitations} \label{app:limitations}
Our characterization is local and holds to leading order in the corruption level; finite corruption, strong curvature, manifold boundaries, and imperfect optimization can therefore mix the two components. The corruption distribution, the branch dimensions $n_\parallel$ and $n_\perp$, and the loss weights $\lambda^\parallel$ and $\lambda^\perp$ must be selected without knowledge of the intrinsic dimension, and the feature geometry requires a separate representation-training stage for each domain. Our experiments concern sparse structured distributions and single-step generators; broader validation across architectures, modalities, and scale remains to be established.

\section{Experimental Detail} \label{app:details}

\subsection{Shared Setup for the Helix, Earth, and Torsion Experiments}

The helix, Earth, and torsion-angle experiments use a common training and evaluation framework. RFM, MF, and RMF use the same network architecture, consisting of four hidden layers of width $256$ and a Fourier time embedding with four frequency octaves. Each method is independently initialized with seed $0$, and evaluation uses exponential-moving-average weights with decay $0.999$. The flow-based methods are evaluated at $\mathrm{NFE}\in\{1,2,5,10\}$, whereas OTM and OTM + \method{} directly generate samples with one network evaluation.

For OTM and OTM + \method{}, we compute an exact Hungarian assignment between generated and real samples using ambient squared distance. Both methods therefore use identical generated-to-data pairings and differ only in their regression losses. For Earth and torsion angles, OTM + \method{} uses
\begin{align}
    \cL_{\mathrm{OTM+\method{}}}
    =
    \|\hat{\bmx}-\bmx_{\pi}\|_2^2
    +0.1\,h_c\!\left(\bmphi^\parallel(\hat{\bmx})
    -\bmphi^\parallel(\bmx_{\pi})\right)
    +0.1\,h_c\!\left(\bmphi^\perp(\hat{\bmx})
    -\bmphi^\perp(\bmx_{\pi})\right),
\end{align}
where $\bmx_{\pi}$ is the assigned real sample and
$h_c(\bmu)=\sqrt{\|\bmu\|_2^2+c^2}-c$ with $c=0.01$. The helix experiment omits the ambient term and uses feature weights $\lambda^\parallel=1$ and $\lambda^\perp=5$.

The feature encoder is trained separately and frozen during generative training. Its parallel decoder predicts the clean sample, while its orthogonal decoder predicts the normalized corruption $\bm\delta/\operatorname{rms}(\bm\delta)$. Normalizing the corruption prevents the orthogonal representation from collapsing when the corruption magnitude is small.

\subsection{Helix on Sphere}

We construct a spherical helix on $\mathbb{S}^2$ with winding parameter $5.0$, latitude oscillation $0.5$, and parameter range $[-\pi,\pi]$. Clean samples are obtained by applying tangent Gaussian perturbations with standard deviation $\sigma_{\mathrm{signal}}=0.01$ and are then embedded into $\mathbb{S}^{512}\subset\mathbb{R}^{513}$ through a fixed random orthonormal map from $\mathbb{R}^3$ to $\mathbb{R}^{513}$. Encoder inputs are formed by applying an additional Gaussian perturbation with per-coordinate standard deviation $\sigma_{\mathrm{perturb}}=0.10$ in the ambient tangent space of $\mathbb{S}^{512}$. The clean distribution therefore remains in the three-dimensional image of the embedding, while the corruption can leave this subspace.

We use $50{,}000$ data samples. All generative models are trained for $50{,}000$ iterations using RAdam with learning rate $10^{-3}$, batch size $256$, gradient clipping at $2.0$, and EMA decay $0.999$. The encoder uses $n_\parallel=3$ parallel and $n_\perp=509$ orthogonal features and is trained with clean- and corruption-prediction weights $1$ and $5$, respectively. OTM and OTM + \method{} use a three-dimensional Gaussian prior and exact Hungarian assignment. For MF, the prior standard deviation is $1/\sqrt{513}$ so that its expected norm matches the unit-norm data.

Evaluation uses $2{,}048$ generated samples. The qualitative comparison in \Cref{fig:qualitative_results} uses $\mathrm{NFE}=1$ and projects samples back to the underlying $\mathbb{S}^2$ coordinates. We additionally monitor sliced $W_2^2$, energy distance, distance to the underlying helix, displacement outside the embedding subspace, and coverage along the helix.

\subsection{Earth and Climate Science}

The Earth benchmark contains event locations on $\mathbb{S}^2$ for volcanoes, earthquakes, floods, and wildfires, with $827$, $6{,}120$, $4{,}875$, and $12{,}809$ records, respectively. We use deterministic $80/10/10$ train, validation, and test splits with seed $42$. Duplicate coordinates are retained during training because repeated events at one location encode the empirical event density.

Clean locations are corrupted in the tangent space of $\mathbb{S}^2$ using Gaussian noise with standard deviation $0.15$. The encoder uses $n_\parallel=n_\perp=2$ and equal clean- and corruption-prediction weights. It is trained for $10{,}000$ iterations with batch size $2{,}048$. Generative models are trained for $20{,}000$ iterations using Adam with learning rate $10^{-3}$ and EMA decay $0.999$. Flow-based methods use batch size $2{,}048$, while the exact-Hungarian OTM variants use batch size $256$.

We report geodesic $W_1$ distance between generated and real event distributions. The cost is the great-circle distance on $\mathbb{S}^2$, and the empirical transport problem is solved by exact Hungarian assignment. Datasets with more than $2{,}048$ records use a fixed seeded reference subsample, shared across all methods and checkpoints. Distances are converted from radians to degrees for \Cref{tab:error}. The qualitative Earthquake comparison in \Cref{fig:qualitative_results} uses $4{,}096$ generated samples at $\mathrm{NFE}=1$.

\subsection{Protein and RNA Torsion Angles}

We use the torsion-angle benchmarks of Riemannian Flow Matching~\citep{chen2024flow}, following the preprocessing of Riemannian diffusion models~\citep{huang2022riemannian,debortoli2022riemannian}. The benchmark contains four two-dimensional Ramachandran distributions and one seven-dimensional RNA torsion distribution:
\begin{center}
\begin{tabular}{lcc}
\toprule
\textbf{Dataset} & \textbf{Manifold} & \textbf{Number of records} \\
\midrule
General  & $\mathbb{T}^2$ & $137{,}718$ \\
Glycine  & $\mathbb{T}^2$ & $13{,}183$ \\
Proline  & $\mathbb{T}^2$ & $7{,}618$ \\
Pre-Pro  & $\mathbb{T}^2$ & $6{,}900$ \\
RNA      & $\mathbb{T}^7$ & $9{,}457$ \\
\bottomrule
\end{tabular}
\end{center}
Each angle is represented by its cosine and sine, embedding $\mathbb{T}^d$ isometrically into $\mathbb{R}^{2d}$ without an angular discontinuity. Corruption is applied independently to each angle using wrapped Gaussian noise with standard deviation $0.15$ radians.

All generative models use batch size $256$ and are trained for $20{,}000$ iterations using Adam with learning rate $10^{-3}$ and EMA decay $0.999$. The encoder is trained separately for $10{,}000$ iterations with batch size $2{,}048$. Its parallel and orthogonal feature dimensions are: $n_\parallel=n_\perp=2$ for $\mathbb{T}^2$ and $n_\parallel=n_\perp=7$ for $\mathbb{T}^7$.

We evaluate geodesic $W_1$ using exact Hungarian assignment under the flat-torus distance, defined as the Euclidean norm of the wrapped per-angle differences. The generated and test distributions are restricted to at most $2{,}048$ samples using a fixed seeded subsample, and the resulting distances are reported in degrees in \Cref{tab:error}. All values in the main table use $\mathrm{NFE}=1$.

\subsection{SCOPe Protein Backbone Generation}

\paragraph{Dataset and generator.}
We use $3{,}673$ SCOPe backbones containing between $60$ and $128$ residues~\citep{chandonia2022scope}, matching the dataset used by RMF~\citep{woo2026riemannianmeanflow}. Each backbone is represented by residue-local rigid frames in $\mathrm{SE}(3)^L$. The generator follows the FrameFlow architecture~\citep{yim2023fast}, using node and edge dimensions $256$ and $128$, six invariant-point-attention blocks, eight attention heads, and self-conditioning.

Both the baseline and \method{} models use the RMF SplitMeanFlow objective, comprising a flow-matching anchor and a semigroup composition term. Rotational and translational losses are weighted by $1$ and $2$, respectively, with translations scaled by $0.1$. The time floor is $0.01$, and both objective terms are clipped at $10$ after $10{,}000$ iterations. Models are trained with AdamW using learning rate $10^{-4}$, weight decay $0.01$, gradient clipping at $1.0$, and EMA decay $0.9999$. Batches contain at most $54$ proteins under a length-dependent memory budget. Although training runs for up to two million iterations, the reported steric evaluation uses the one-million-step checkpoints.

\paragraph{Frozen protein encoder.}
The protein feature encoder is a split-latent U-Net over backbone frames with $64$ feature planes, encoder depths $[4,4,4,4]$, and decoder depths $[2,2,2]$. It is trained for $250{,}000$ iterations using batch size $2{,}048$, learning rate $10^{-4}$, and equal clean- and corruption-prediction weights. Structured corruptions include rigid segment and loop displacements, chain breaks, and atom-level bond, angle, and torsion perturbations. Each of the $13$ corruption operations is independently activated according to a base probability of $0.15$, producing approximately two corruptions per sample. The EMA encoder checkpoint is frozen during RMF training.

The baseline sets the feature-semigroup weight to zero. The \method{} model adds the decomposed feature semigroup term with weight $0.0029$ while retaining the standard RMF objective. All other architecture, optimization, conditioning, and sampling settings are identical between the two runs.

\paragraph{Designability and self-consistency.}
For each checkpoint, backbone length, and NFE, we generate $64$ backbones from identical priors for both models. Following standard protein-backbone evaluation, ProteinMPNN~\citep{dauparas2022robust} generates eight sequences per backbone at temperature $0.1$, which are refolded using ESMFold~\citep{lin2023evolutionary}. Self-consistency RMSD is the minimum RMSD over the eight refolded sequences, and a backbone is considered designable when this value is below $2.0$~\AA.

\paragraph{Steric validity.}
We evaluate the one-million-step checkpoints at lengths
\[
\{60,68,76,84,92,100,108,116,124,128\}
\]
and $\mathrm{NFE}\in\{1,5,10\}$. For each length and NFE, we draw eight independent batches of $128$ samples, using identical priors for the baseline and \method{} within each batch.

For every generated structure, we compute the maximum van der Waals overlap
\begin{align}
    \max_{(a,b)}\left(r_a+r_b-\|\bmx_a-\bmx_b\|_2\right)
\end{align}
over backbone $\mathrm{N}$, $\mathrm{C}_{\alpha}$, and $\mathrm{C}$ atoms belonging to residues separated by at least two sequence positions. We use Bondi radii of $1.55$~\AA{} for nitrogen and $1.70$~\AA{} for carbon~\citep{bondi1964van}.
Backbone oxygen atoms are excluded because their placement is determined by the PDB-writing convention rather than directly emitted by the model.
The reported statistic is the median maximum overlap within each batch.
Real SCOPe backbones within four residues of each target length provide the reference distribution and have median overlaps of $0.28$--$0.30$~\AA. 

\end{document}